\documentclass{article}

  \usepackage[preprint]{neurips_2025}

\usepackage[utf8]{inputenc} 
\usepackage[T1]{fontenc}    
\usepackage{hyperref}       
\usepackage{url}            
\usepackage{booktabs}       
\usepackage{amsfonts}       
\usepackage{nicefrac}       
\usepackage{microtype}      
\usepackage{xcolor}         

\usepackage{amsthm}
\usepackage{amsmath}
\usepackage{graphicx}

\usepackage{algorithm}
\usepackage{algpseudocode}
\usepackage{comment}

\theoremstyle{definition}

\newtheorem{assumption}{Assumption}
\newtheorem{example}{Example}
\newtheorem{theorem}{Theorem}

\newtheorem{lemma}[theorem]{Lemma}

\title{Bootstrapping Diffusion: Diffusion Model Training Leveraging Partial and Corrupted Data}

\author{Xudong Ma, San Francisco, California, xma@ieee.org}

\begin{document}

\maketitle

\begin{abstract}
	
Training diffusion models requires large datasets. However, acquiring large volumes of high-quality data can be challenging, for example, collecting large numbers of high-resolution images and long videos. On the other hand, there are many complementary data that are usually considered corrupted or partial, such as low-resolution images and short videos. Other examples of corrupted data include videos that contain subtitles, watermarks, and logos. In this study, we investigate the theoretical problem of whether the above partial data can be utilized to train conventional diffusion models. Motivated by our theoretical analysis in this study, we propose a straightforward approach of training diffusion models utilizing partial data views, where we consider each form of complementary data as a view of conventional data. Our proposed approach first trains one separate diffusion model for each individual view, and then trains a model for predicting the residual score function. We prove generalization error bounds, which show that the proposed diffusion model training approach can achieve lower generalization errors if proper regularizations are adopted in the residual score function training. In particular, we prove that the difficulty in training the residual score function scales proportionally with the signal correlations not captured by partial data views. Consequently, the proposed approach achieves near first-order optimal data efficiency.

\end{abstract}
\section{Introduction}
\label{sec:intro}

Diffusion models are state-of-the-art generative deep learning models, such as image generation \citep{dhariwal2021diffusion}, \citep{rombach2022high}, \citep{gu2023pixart}, \citep{song2022versatile}, video generation \citep{ho2022video}, \citep{openai2025sora}, \citep{modelscope2023}, \citep{animatediff2023}, and 3D generation \citep{lin2022magic3d}, \citep{poole2022dreamfusion}, \citep{chen2023fantasia3d}, \citep{zhou2021pointvoxel}, \citep{cheng2023sdfusion}. However, modern diffusion models face significant challenges; they require extensive datasets for effective training. State of the art image generation diffusion models are typically trained on image datasets containing billions of images \citep{schuhmann2022laion5b}, \citep{kakaobrain2023coyo}. State of the art video datasets usually contain 10 to 100 millions of videos \citep{bain2021webvid}, \citep{xue2022hdvila}. However, there remains a strong demand for large, high-quality image and video datasets, particularly those that feature high-resolution and long-duration videos. It is well known from scaling laws (e.g., \citep{liang2024scaling}) that the generation performance can be further improved by scaling both model sizes and training dataset sizes. Additionally, there is a significant demand for high-resolution image and video data in specialized domains, such as long-tail scenarios in autonomous driving and robotic perception in manufacturing and assembly environments.

In this paper, we consider a scenario in which a relatively small dataset of full-resolution images and videos is available, while a substantially larger dataset consists of partially observed or corrupted data. To address the above data demand, we explore an approach of leveraging the so-called partial or corrupted data for diffusion training. For training diffusion models for high-resolution images or long video generation, there are usually far more low-resolution images and short video data available. Such low-resolution and short data may still be useful for diffusion model training if we consider them partial observations of the required data. In other words, low-resolution images and short videos are just down-sampled versions of the required data. In addition, many real-world images and video data may contain artifacts such as subtitles, watermarks, and logos. These elements make such data unsuitable for diffusion model training because it is undesirable for models to learn and reproduce these non-content features during generation. However, patches from these image and video data may still be useful because they can also be considered as partial observations of the required data.

In this study, we attempt to answer the fundamental question of whether partial view data can be leveraged for training a conventional full-resolution diffusion model. Here, we refer to each form of incomplete observation as a partial view. For instance, a low-resolution image can be regarded as a partial view of the high-resolution image on which the diffusion model would be trained. Likewise, an image patch may also be interpreted as a specific partial view of the full-resolution image.

Our study provides an affirmative answer to this question by demonstrating that partial-view data can be utilized to effectively mitigate the challenges associated with the stringent data requirements of training the full-resolution diffusion model. We organize our presentation such that we first present a straightforward approach for training diffusion models utilizing partial-view data; subsequently, we offer a theoretical analysis of the proposed approach. We call our proposed approach bootstrapping diffusion.

In the first stage of our proposed approach, we independently train a diffusion model for each distinct data view. Specifically, we learn a denoiser neural network that estimates the expectations of the clean partial view data from its corresponding noisy observation following standard diffusion model training. The approach then aggregates the partial-view expectations into one expectation for full-resolution data, for example, via a linear combination of the individual view-wise expectations. In the subsequent stage, we train a residual denoiser to predict the discrepancy between the ground-truth expectation and the aggregated expectation. Motivated by our theoretical analysis, we adopt variance regularization on the residual denoiser. The $l$2 norm of the output of the residual denoiser either penalized or explicitly scaled to match the true signal statistics.

We prove the generalization error bounds for the cases of training without partial-view data and training with partial-view data via the bootstrapping approach. Based on these generalization error bounds, we demonstrate that the proposed approach achieves lower generalization errors than the approach without training with partial-view data. We also demonstrate that the difficulty of training the residual denoiser is proportional to the amount of global signal correlation that cannot be captured by partial-data views. In particular, bootstrapping diffusion is near first-order optimal in terms of data efficiency. Finally, We present experimental results on AFHQv2-Cat dataset.

The main contributions of this paper are summarized as follows:
\begin{itemize}
\item We propose a novel approach to training diffusion models that effectively leverages partially observed and corrupted data.

\item We provide a theoretical analysis of the generalization error bounds for the proposed method, demonstrating that the generalization error can be reduced by applying  variance regularizations to the residual denoiser.

\item We present experimental results that validate our theoretical findings and demonstrate the practical effectiveness of our approach.
\end{itemize}
\section{Related Work}

Training diffusion models on noisy data to recover clean data distributions has been previously discussed in the literature. In \citep{daras2023ambient} and \citep{daras2024consistent}, an ambident diffusion approach was proposed, where additional distortions were introduced so that the diffusion model could predict the original corrupted images from its further corrupted version. In \citep{kawar2023gsure}, a GSURE-diffusion approach was proposed, in which Stein's lemma was used to avoid computing MSEs using clean data samples directly.  In \citep{bai2024expectation}, an expectation-maximization-based approach was used, where in the E-steps, a model initialized from a clean image dataset was used to generate samples, and in the M-step, the log likelihood of the generated samples was maximized.

Several notable distinctions exist between our proposed approach and prior methods in the literature. Existing approaches generally assume that partial views contain sufficient information to effectively train diffusion models for full-resolution data. In contrast, our study focuses on scenarios in which partial views lack certain essential information, necessitating the use of a complementary full-resolution dataset to effectively train the diffusion model at full resolution. Moreover, the discussions in this paper primarily address the theoretical question of whether partial-view data can aid in training diffusion models at full resolution, and to what extent data efficiency improvements can be achieved.

\section{Preliminary}

\subsection{Diffusion Models}
Drawing inspiration from non-equilibrium statistical physics, diffusion models are a flexible and tractable framework for sampling from complex data distributions
\citep{sohl2015deep}, \citep{ho2020denoising}. Diffusion models adopt a forward process to gradually destroy structures in data distributions and then learn a reverse process to recover structures from pure noisy distributions. In continuous-time formulation, both the forward and reverse processes can be described by stochastic differential equations (SDEs) \citep{song2021scorebased}. The SDE for the forward process can be written as
\begin{align}
dx = f(x_t,t)dt + g(t)dw  \nonumber
\end{align}
where, $x_t$ is the continuous time variables with $t\in [0, T]$, $x_0$ is the original data sample and $x_T$ is the sample of prior distribution, $w$ is the standard Wiener process, $f(x_t, t)$ is the drift coefficient and $g(t)$ the diffusion coefficient. The reverse process SDE is \citep{anderson1982reverse},
\begin{align}
dx = \left[f(x_t, t) - g(t)^2 \nabla_{x_t} \log p_t(x_t)\right]dt + g(t)\bar{w}
\end{align}
where $\bar{w}$ is another standard Wiener process, $p_t$ is the probability distribution for $x_t$. The term $\nabla_{x_t} \log p_t(x_t)$ is known as the score function. By Tweedie's formula \citep{efron2011tweedie},
\begin{align}
\nabla_{x_t} \log p_t(x_t)=\frac{E[x_0|x_t] - x_t}{\sigma_t^2}
\end{align}
where $x_t = x + \mathcal{N}(0, \sigma_t^2 I)$.Training diffusion models essentially reduces to estimating score functions. Several equivalent but distinct approaches exist for this estimation. In this paper, we consistently use deep neural networks to predict the conditional expectation 
$E[x_0 \mid x_t]$ given a noisy sample $x_t$
and the corresponding time step $t$.

\subsection{Covering numbers}
\label{subsec:covering_number}

Covering numbers provide a metric for quantifying the complexity of neural network hypothesis classes. Specifically, a covering number indicates how many balls of a small radius $\epsilon$ are necessary to entirely cover the set of functions within a hypothesis class 
$\mathcal{F}$. In this paper, we primarily adopt the definition of covering numbers from \citet{graf2022excess}, where the distance $d(f_1, f_2)$ between two functions 
$f_1,f_2 \in \mathcal{F}$ is defined with respect to the training dataset 
$\{x_1,x_2,\ldots,x_N\}$. Concretely, this definition is given by:
\begin{align}
d(f_1, f_2) = \sqrt{\frac{1}{N} \sum_{n=1}^{N} \|f_1(x_n) - f_2(x_n)\|^2}
\end{align}
We use $\mathcal{N}(\mathcal{F}, \epsilon, d)$ to denote the covering number.

For standard neural networks with bounded parameters and activation functions (such as ReLU), it is established in the literature \citep{graf2022excess} that covering numbers typically scale exponentially with network depth and polynomially with network width and the inverse radius $1/ \epsilon$. Additionally, covering numbers depend on the Lipschitz constant of the neural network, increasing as the Lipschitz constant grows. While exact covering numbers may vary among different neural network architectures, we follow the general assumptions discussed in prior works, such as \citet{graf2022excess}, summarized as follows.

\begin{assumption}
\label{assumption:cover_num}
We assume the covering number of the function class according to our neural network is upper bounded by
\begin{align}
\log \mathcal{N}(\mathcal{F}, \epsilon, d) \leq \bar{L} W \log\left(1+\bar{L} C \frac{N}{\epsilon}\right)
\end{align}
\begin{itemize}
	\item $\bar{L}$ is a value that is mainly determined by the depth of the neural network
	\item $W$ is the largest total number of parameters within one network layer, for example, if the convolutional layer has $c_{in}$ input channels, $c_{out}$ output channels, and kernel size $(k_h, k_w)$, then the total number of parameters within this layer is $W = c_{in} \times c_{out} \times k_h \times k_w$
	\item  $C$ is a parameter mainly determined by the Liptschiz of the neural network, whereas $C$ increases when the Liptschiz number increases. When the Liptschize of the neural network goes to zero, $C$ goes to zero, thus the covering number  $\mathcal{N}(\mathcal{F}, \epsilon, d)$ goes to 1.
\end{itemize}
\end{assumption}

\begin{example}
Theorem 3.4. in \citep{graf2022excess} provides upper bounds for covering numbers of residual networks. The upper bound in Equation (15b) satisfies the assumption \ref{assumption:cover_num}.
\end{example}

\subsection{Notation and additional assumptions}

\begin{assumption}
\label{assumption:bounded}
We assume that each element of the clean data samples $x_0$ is bounded within the interval $[-U,U]$, where $U>0$. Under this assumption, it follows that the conditional expectation $E[x_0 \mid x_t]$ is also bounded element-wise by the same interval $[-U,U]$. Let $m$ represent the dimension of $x_0$; this implies that the $l_2$-norm satisfies 
$\|x_0\|_2 \leq \sqrt{m}U$. Consequently, we also assume that any denoiser used to estimate $\mathbb{E}[x_0 \mid x_t]$ is similarly bounded element-wise within $[-U,U]$.
\end{assumption}

Throughout this paper, random vectors are denoted by boldface uppercase letters, while their realizations are represented by the corresponding boldface lowercase letters.

\section{Bootstrapping Diffusion}
\label{sec:bootstrap}

Before proceeding into further technical details, we first establish the required notation. Let us denote the full-resolution dataset as $\mathcal{S}_0 = \{x_{1,0}, x_{2,0}, \ldots, x_{n,0}, \ldots, x_{N_0, 0}\},$
with cardinality $|\mathcal{S}_0| = N_0$. Additionally, we assume the availability of $I$ supplementary datasets, labeled as $\mathcal{S}_1, \mathcal{S}_2, \ldots, \mathcal{S}_I.$
We assume all data samples across these datasets share the same dimension $m$ and originate from the same underlying probability distribution $p_0$. Within each individual dataset, samples are independently and identically distributed (IID). Furthermore, dataset $\mathcal{S}_0$ has no intersection with the other datasets, whereas datasets $\mathcal{S}_1, \ldots, \mathcal{S}_I$ may contain duplicated samples. We directly observe only the dataset $\mathcal{S}_0$. For each of the datasets $\mathcal{S}_i$ with $i = 1,\ldots,I$, observations are only available through specific data views. In particular, there exists a projection matrix $A_i$ corresponding to each dataset $\mathcal{S}_i$, such that we can only observe the projected samples 
$A_i x, \quad \text{for each } x \in S_i.$
For example, the matrix $A_i$ can represent a downsampling operation that transforms a $256 \times 256$ image into a lower-resolution $64 \times 64$ image. Alternatively, $A_i$ may serve as an image-cropping operator, extracting a $64 \times 64$ patch from the original full-sized $256 \times 256$ image. Additionally, $A_i$ could act as a video-clipping operator that reduces a longer sequence of video frames into a shorter subsequence.

A naive approach to training diffusion models for generating data samples drawn from the distribution $p_0$ is to utilize only the dataset $\mathcal{S}_0$, while discarding the additional datasets $\mathcal{S}_1, \ldots, \mathcal{S}_I$. In Section~\ref{sec:theoretical}, we provide generalization error bounds for this scenario. However, due to the limited size of dataset $\mathcal{S}_0$, this naive strategy carries a high risk of overfitting to the small training set, a phenomenon further illustrated by the generalization error bounds established in Section~\ref{sec:theoretical}.

For the proposed bootstrapping diffusion method, the initial step involves training a separate denoiser for each data view and the associated dataset $\mathcal{S}_i$, with $i = 1, \ldots, I$. Each denoiser is trained to estimate the conditional expectation $\mathbb{E}[A_i \mathbf{X}_0 \mid A_i \mathbf{x}_t]$, where $\mathbf{X}_0$ denotes a full-resolution data sample, and $\mathbf{x}_t$ is its noisy observation at diffusion step $t$. Although we do not directly observe the original samples $\mathbf{x} \in \mathcal{S}_i$, we can access their projected representations $A_i\mathbf{x}$. Consequently, we can still train the denoiser defined as $f_i(A_i \mathbf{x}_t, t; \theta_i) = \mathbb{E}[A_i x_0 \mid A_i \mathbf{x}_t]$, where $\theta_i$ represents the model parameters.

Thus, in the second step, we have a well-trained denoiser $f_i(\cdot)$ for each individual data view. Given a data pair consisting of $x_0 \in S_0$ and its noisy counterpart $x_t$ at diffusion step $t$, we can produce a combined estimate of $\mathbb{E}[x_0 \mid x_t]$—for example, by employing a linear combination of predictions from each data-view denoiser, given by $\sum_{i=1}^{I} B_i f_i(A_i x_t; \theta_i)$. Consequently, we train a residual denoiser $f_0(x_t; \theta_0)$ by minimizing the loss function
\begin{align}
\mathcal{L} = \mathbb{E}_{x_0, t, \epsilon}\left[\left\|x_0  - f_0(x_t) - \sum_{i=1}^{I} B_i f_i(A_i x_t)\right\|^2\right]
\end{align}

According to the analysis in Section~\ref{sec:theoretical}, the variances of the residual denoiser $f_0(\cdot; \theta_0)$ can be much smaller than the variances of the other denoisers $f_i, i=1, \ldots, I$. It is therefore preferred to apply a variance regularization for training the residual denoiser $f_0(\cdot; \theta_0)$,
\begin{align}
\sum_{n=0}^{N_0} \|f_0(x_{n, t}; \theta_0)\|_2^2 \leq M 
\end{align}
where the motivation for this variance regularization is clear after the theoretical discussion in Section \ref{sec:theoretical}. We can also introduce a Lagrange multiplier and convert the constrained optimization problem into an unconstrained optimization problem. Thus, in the second step, we can train the residual denoiser $f_0(\cdot; \theta_0)$ to minimize 
\begin{align}
\label{eq:bootstrapping_loss}
\mathcal{L} = \mathbb{E}_{x_0, t, \epsilon}\left[\left\|x_0  - f_0(x_t) - \sum_{i=1}^{I} B_i f_i(A_i x_t)\right\|^2\right] +  \frac{\lambda}{N_0}\sum_{n=0}^{N_0} \|f_0(x_t; \theta_0)\|_2^2 
\end{align}
There is another network architecture choice, in which the denoiser $f_0(\cdot; \theta_0)$ may consist of one range adapter $s(t)$ and one conventional neural network $g(\cdot; \theta'_0)$. That is, $f_0(\cdot;\theta_0) = s(t)g(\cdot;\theta'_0)$, where $s(t)$ is a scalar multiplier that adjusts the dynamic range of the diffusion denoiser $f_0(\cdot; \theta_0)$ according to the different time steps $t$. The overall bootstrapping diffusion training is summarized in Algorithm \ref{algorithm:bootstrapping}.
At inference time, we compute the expectations using $f_0(x_t) + \sum_{i=1}^{I} B_i f_i(A_i x_t)$

\begin{algorithm}[htbp]
	\label{algorithm:bootstrapping}
	\caption{Bootstrapping Diffusion}
	\begin{algorithmic}[1]
		\Require Input full-resolution dataset $S_0$, $I$ partial view datasets $S_1, \ldots, S_I$ and the corresponding projection matrices $A_i$
		\Ensure Output denoisers $f(\cdot; \theta_0)$, $\ldots$, $f(\cdot; \theta_I)$ 
		\For{$n = 1$ to $I$}
		\State Train diffusion denoiser $f(\cdot; \theta_i)$ on the data samples $\{A_i x \mid x\in S_i\}$
		\EndFor
		\State Train diffusion denoiser $f_0(\cdot; \theta_0)$ on the data samples $\mathcal{S}_0$ by minimizing the loss function in Equation \ref{eq:bootstrapping_loss}
	\end{algorithmic}
\end{algorithm}

\section{Theoretical Analysis}
\label{sec:theoretical}

As discussed in previous sections, our main theoretical objective is to understand how generalization error bounds scale with respect to the number of full-resolution and partial-view data samples. In Section~\ref{sec:bootstrap}, we introduced bootstrapping diffusion. The motivations underlying the design of bootstrapping diffusion are largely guided by the theoretical analysis presented in this section.

Throughout this paper, our analysis primarily focuses on generalization error bounds for diffusion denoisers. In fact, previous literature has established a connection between the KL divergence of the true data distribution and the generated data distribution, and the generalization error of diffusion denoisers. Consequently, our analysis of generalization error bounds for diffusion denoisers can be directly translated into corresponding bounds on KL divergence. We will elaborate further on this connection in Subsection~\ref{subsec:connection_diffusion_denoiser}.

In Subsection~\ref{subsec:bound_full_resolution_only}, we derive the generalization error bounds for diffusion denoisers under the most general scenario. We apply the same bounding technique as in Subsection~\ref{subsec:scaling_residual} to establish generalization error bounds specifically for the residual diffusion denoiser case. Additionally, in Subsection~\ref{subsec:residual_variance}, we demonstrate that the residual denoiser generally exhibits low output variance. Based on the above discussion, we discuss how the generalization error for the diffusion denoisers scale and the motivation for the variance regularization. Essentially, variance regularization limits the Lipschitz of the residual diffusion denoisers and, in turn, limits the model complexities of the residual denoiser. We demonstrate that this regularization of model complexities should be adapted to the training difficulty of residual denoisers. This model complexity adaptation makes it possible to achieve near first-order optimality in terms of data efficiency.

\subsection{Connections between the generalization errors for denoisers and the KL distance}
\label{subsec:connection_diffusion_denoiser}

As established in previous work (see Equation~10 of \citet{minde2024}), given two score functions $s_t^{\mu^A}$ and $s_t^{\mu^B}$ with corresponding data distributions $\mu^A$ and $\mu^B$, the KL divergence between these distributions can be expressed in terms of the score functions as follows:
\begin{align}
\label{eq:kl_diffusion_denoiser}
KL\left[\mu^A \,\|\, \mu^B \right]= \mathbb{E}_{\mathcal{P}^{\mu^A}} \left[
\int_{0}^{T} \frac{g(t)^2}{2} \lVert  s_t^{\mu^A} (X_t) - s_t^{\mu^B}(X_t)
\rVert^2 dt \right] + KL\left[\nu_T^{\mu^A} \,\|\, \nu_T^{\mu^B}\right]
\end{align}
where $\nu_T^{\mu^A}$ and $\nu_T^{\mu^B}$ represent the distributions of $x_t$ at the terminal time $T$ of the forward diffusion process, and $\mathbb{E}_{\mathcal{P}^{\mu^A}}$ indicates that the expectation is taken with respect to the distribution $\mu^A$ and the corresponding path measures.

It should be noted that the distributions $\nu_T^{\mu^A}$ and $\nu_T^{\mu^B}$ should both be approximately Gaussian with the same mean and variance for sufficiently large $T$. Therefore, the KL distance between the two distributions should be close to zero. Thus, the KL distances between the two data distributions $\nu_T^{\mu^A}$ and $\nu_T^{\mu^B}$ are primarily determined by the differences between their corresponding score functions.

If we consider $\mu^A$ as the ground-truth data distribution, $s_t^{\mu^A}$ is the ground-truth score function, $s_t^{\mu^B}$ is the score function by the trained diffusion model, then the KL distance between the ground-truth data distribution to the distribution of data samples generated by the diffusion model at inference time would be determined by the differences between the ground-truth score function and the trained score function. Furthermore, due to Assumption~\ref{assumption:bounded}, the score functions considered here satisfy the conditions required by the dominated convergence theorem. Consequently, minimizing the generalization errors of diffusion denoisers ensures convergence of the KL divergence.

\subsection{Generalization error bound for general cases}
\label{subsec:bound_full_resolution_only}

In this subsection, we prove a generalization error bound for the diffusion denoiser $f(\mathbf{x}_t, t; \theta_0)$ when only the full-resolution dataset $\mathcal{S}_0$ is used for training. The normalized loss function is denoted by $\widehat{L}(\theta)$
\begin{align}
\label{eq:objective_full_resolution_only}
\widehat{L}(\theta) = \frac{1}{NK}\sum_{k=1}^{K} \sum_{n=1}^{N} \left\| f(\mathbf{x}_{n,k}, t_{n,k}; \theta) - \mathbf{x}_{n,0} \right\|_2^2
\end{align} 
where $K$ is the number of epochs, $N$ is the cardinality of the training set $S_0$, $x_{n,0}\in S_0$ is the nth data sample in the dataset $S_0$, $t_{n, k}$ is a randomly sampled diffusion time step, and $\mathbf{x}_{n,k}$ denotes a noisy version of $\mathbf{x}_{n,0}$ at time step $t_{n,k}$. We assume that the data samples in the training dataset $S_0$ are drawn independently and identically distributed (IID) from distribution $p_0$. For each data sample $\mathbf{x}_n \in S_0$, the diffusion time steps $t_{n, 1}, \ldots, t_{n, K}$ were drawn independently and identically. For each $\mathbf{x}_{n, 0}$ and $t_{n ,k}$, noisy $\mathbf{x}_{n, k}$ is drawn independently. We use $L(\theta)$ to denote
\begin{align}
\label{eq:objective_full_resolution_only}
L(\theta) = \mathbb{E} \left[\left\| f(\mathbf{x}_{n,k}, t_{n,k}; \theta) - \mathbf{x}_{n,0} \right\|_2^2\right]
\end{align}

However, the loss function given in Equation~\ref{eq:objective_full_resolution_only} serves only as a proxy; our actual objective is to minimize the true loss function defined as follows:
\begin{align}
R(\theta) =  \mathbb{E}\left[\left\|f(\mathbf{x}_{n,k}, t_{n,k}; \theta) - \mathbb{E}[\mathbf{X}_{ n,0}\mid \mathbf{x}_{n,k}, t_{n,k}] \right\|_2^2 \right]
\end{align}
\begin{align}
\widehat{R}(\theta) = \sum_{n=1}^{N} \sum_{k=1}^{K} \frac{1}{NK}\left\|f(\mathbf{x}_{n,k}, t_{n,k}; \theta) - \mathbb{E}[\mathbf{X}_{ n,0}\mid \mathbf{x}_{n,k}, t_{n,k}] \right\|_2^2  \nonumber 
\end{align}
We use $\mathcal{V}(\mathcal{S})$ to denote the prediction variance
\begin{align}
\mathcal{V}(\mathcal{S}) = \frac{1}{NK}\sum_{n=1}^{N}\sum_{k=1}^{K}\left\|E[\mathbf{X}_{n,0} \mid \mathbf{x}_{n,k}, t_{n,k}] - \mathbf{x}_{n,0} \right\|_2^2 
\end{align}

Let $\mathcal{F}$ be the class of diffusion denoisers $f(\cdot; \theta)$, $\theta$ are model parameters from a parameter set $\theta \in \Theta$. We use $\Delta_b$ to denote the following bias of the function class $\mathcal{F}$,
\begin{align}
 \Delta_b^2 = \min_{\theta \in \Theta} R(\theta) 
\end{align}
We use $\theta^\ast$ to denote one of the minimizers of the above function.
Assume that $\mathcal{F}$ can be covered by a set $\epsilon$-balls $\mathcal{B}$, where  $|\mathcal{B}| = \mathcal{N}(\mathcal{F}, \epsilon, \|\|_2)$. We use $\mathcal{C}$ to denote the set of centers of the $\epsilon$-balls in $\mathcal{B}$. Let us divide $\mathcal{C}$ into two parts
\begin{align}
& \mathcal{C}_1 = \left\{\theta \in \mathcal{C} \mid R(\theta) \geq \left(\sqrt{\mathbb{E}[\mathcal{V}(\mathcal{S}_0)] + \Delta_b^2 + \Delta_v^2} + \epsilon\right)^2 - \mathbb{E}[\mathcal{V}(\mathcal{S}_0)]+ \rho \right\} \nonumber \\
& \mathcal{C}_2 = \mathcal{C} \setminus \mathcal{C}_1 \nonumber
\end{align}
Let $\mathcal{E}_1$ denote the random event that 
\begin{align}
\mathcal{L}(\theta^\ast) \geq \left(\mathbb{E}[\mathcal{V}(\mathcal{S}_0)] + \Delta_b^2 + \Delta_v^2\right)NK
\end{align}
Let $\mathcal{E}_2$ denote the random event that for one of the $\widehat{\theta}\in \mathcal{C}_1$, 
\begin{align}
\mathcal{L}(\widehat{\theta})  \leq \left( \sqrt{\mathbb{E}[\mathcal{V}(\mathcal{S}_0)] + \Delta_b^2 + \Delta_v^2} + \epsilon\right)^2 NK
\end{align}

\begin{theorem}
	\label{theorem:bound_e1}
	The probability $\Pr(\mathcal{E}_1)$ is upper bounded by
\begin{align}
\Pr(\mathcal{E}_1) \leq \exp\left(\frac{-2 \Delta_v^2 NK}{(64+16K) m^2U^4}\right)
\end{align}
\end{theorem}
\begin{proof}
	A detailed proof is provided in Appendix \ref{proof:bound_e1}.
\end{proof}

\begin{theorem}
\label{theorem:bound_e2}
The probability $\Pr(\mathcal{E}_2)$ is upper bounded by
\begin{align}
\Pr[\mathcal{E}_2] \leq \mathcal{N}(\mathcal{F}, \epsilon, d)\exp\left(\frac{-2 \rho^2 NK }{(64+16K) m^2U^4}\right)
\end{align}
\end{theorem}
\begin{proof}
	A detailed proof is provided in Appendix \ref{proof:bound_e2}.
\end{proof}

Radamacher complexity is a well-known metric for machine learning model complexity \citet{BartlettMendelson2002}.
Here, We define two Radamacher complexities $\mathfrak{R}_L(\mathcal{F}, \mathcal{S}_0)$, and $\mathfrak{R}_R(\mathcal{F}, \mathcal{S}_0)$,
\begin{align}
\mathfrak{R}_L(\mathcal{F}, \mathcal{S}_0) = 
\mathbb{E}_{\mathcal{S}_0} \left[ \sup_{\theta \in \Theta} \frac{1}{NK} \sum_{n=1}^{N}\sum_{k=1}^{K} \sigma_{n,k}  \ell(\mathbf{X}_{n,k}, t_{n,k}; \theta)  \right] 
\end{align}
where $\sigma_{n,k}$ are independently uniformly distributed random variables taking values in $\{-1, +1\}$ and
\begin{align}
\ell(\mathbf{X}_{n,k}, t_{n,k}; \theta) =  \left\|f(\mathbf{X}_{n,k}, t_{n,k}; \theta) - \mathbf{X}_{ n,0} \right\|_2^2
\end{align}
\begin{align}
\mathfrak{R}_R(\mathcal{F}, \mathcal{S}_0) = 
\mathbb{E}_{\mathcal{S}_0} \left[ \sup_{\theta \in \Theta} \frac{1}{NK} \sum_{n=1}^{N}\sum_{k=1}^{K} \sigma_{n,k}  \ell(\mathbf{X}_{n,k}, t_{n,k}; \theta)  \right] 
\end{align}
where we slightly abuse the notation
\begin{align}
\ell(\mathbf{X}_{n,k}, t_{n,k}; \theta) =  \left\|f(\mathbf{X}_{n,k}, t_{n,k}; \theta) - \mathbb{E}[\mathbf{X}_{ n,0}\mid \mathbf{X}_{n,k}, t_{n,k}] \right\|_2^2
\end{align}

Let $\mathcal{E}_3$ denote the random event that 
\begin{align}
\sup_{\theta} \left| L(\theta) - \widehat{L}(\theta) \right | \geq 2\mathfrak{R}_L(\mathcal{F}, \mathcal{S}_0) + \gamma
\end{align}

\begin{theorem}
	\label{theorem:bound_full_resolution}
	Unless a rare event occurs, the minimizer of the loss function $\mathcal{L}(\theta)$ achieves a generalization error
	\begin{align}
	R(\theta_2) \leq & \left(\sqrt{\left( \left(\sqrt{\mathbb{E}[\mathcal{V}(\mathcal{S}_0)] + \Delta_b^2 + \Delta_v^2} + \epsilon\right)^2 + \rho + 2 \mathfrak{R}_L(\mathcal{F}, \mathcal{S}_0) + \gamma \right)} + \epsilon \right)^2 \nonumber \\
	& + 2 \mathfrak{R}_L(\mathcal{F}, \mathcal{S}_0) + \gamma - \mathbb{E}[\mathcal{V}(\mathcal{S}_0)] \nonumber
	\end{align}
	where the probability of the rare event is at most 
	\begin{align}
	& \exp\left(\frac{-2 \Delta_v^2 NK}{(64+16K) m^2U^4}\right) \nonumber \\
	& + \mathcal{N}(\mathcal{F}, \epsilon, d)\exp\left(\frac{-2 \rho^2 NK }{(64+16K) m^2U^4}\right) + \exp\left(\frac{-\gamma^2N}{32m^2U^4(1+1/K)}\right) \nonumber
	\end{align}
\end{theorem}
\begin{proof}
The proof is in Appendix~\ref{proof:bound_full_resolution}
\end{proof}

\subsection{Variances of residual diffusion denoisers}
\label{subsec:residual_variance}

Let $g(\cdot; \theta_1, \ldots, \theta_I)$ denote the combined denoiser from partial data view denoisers. Using Tweedie's Formula \citet{efron2011tweedie}, we know the corresponding score function
\begin{align}
s(\mathbf{x}_t; \theta_1, \ldots, \theta_I) & = \nabla_{\mathbf{x}_t}\log p_t(\mathbf{X}_t) = \frac{\mathbb{E}[\mathbf{X}_0 \mid \mathbf{x}_t] - \mathbf{x}_t}{\sigma_t^2}  = \frac{\sum_{i=1}^I B_i f(A_i \mathbf{x}_t, t ; \theta_i) - \mathbf{x}_t}{\sigma_t^2}
\end{align}
Using Equation~\ref{eq:kl_diffusion_denoiser}, the $KL$-distance between the true data distribution and generated distribution using score functions $s(\cdot; \theta_1, \ldots, \theta_I)$
\begin{align}
KL\left[\mu^A \,\|\, \mu(\theta_1, \ldots, \theta_I) \right] & \approx  \mathbb{E}_{\mathcal{P}^{\mu^A}} \left[
\int_{0}^{T} \frac{g(t)^2}{2} \lVert  s_t^{\mu^A} (\mathbf{X}_t) - s_t(\theta_1, \ldots, \theta_I)(\mathbf{X}_t)
\rVert^2 dt \right] \nonumber \\
& = \mathbb{E}_{\mathcal{P}^{\mu^A}} \left[
\int_{0}^{T} \frac{g(t)^2}{2} \lVert  r (\mathbf{X}_t) 
\rVert^2 dt \right] \nonumber 
\end{align}
where $r(\mathbf{X}_t)$ is exactly the training goal that the residual denoiser try to learn. Given that $s_t(\theta_1, \ldots, \theta_I)$ is trained from partial data views to approximate the true score function $s_t^{\mu^A}$, $KL\left[\mu^A \,\|\, \mu(\theta_1, \ldots, \theta_I) \right]$ should take small values and thus the training goal $r(\mathbf{X}_t)$ should have small variances. In fact, some discussion from another perspective is also provided in Appendix~\ref{appendix:residual_variance}.

\subsection{Generalization bounds for residual diffusion denoisers}

\label{subsec:scaling_residual}

For the residual denoiser, we can prove a generation error bounds in Theorem~\ref{theorem:bound_residual} using similar arguments in Theorem~\ref{theorem:bound_full_resolution}. However, the Radamacher definitions are different and the detailed definitions can be found in  Appendix~\ref{proof:bound_residual}.

\begin{theorem}
	\label{theorem:bound_residual}
	Unless a rare event occurs, the minimizer of the loss function $\mathcal{L}(\theta)$ achieves a generalization error
	\begin{align}
	R(\theta_2) \leq & \left(\sqrt{\left( \left(\sqrt{\mathbb{E}[\mathcal{V}(\mathcal{S}_0)] + \Delta_b^2 + \Delta_v^2} + \epsilon\right)^2 + \rho + 2 \mathfrak{R}_L(\mathcal{F}, \mathcal{S}_0) + \gamma \right)} + \epsilon \right)^2 \nonumber \\
	& + 2 \mathfrak{R}_L(\mathcal{F}, \mathcal{S}_0) + \gamma - \mathbb{E}[\mathcal{V}(\mathcal{S}_0)] \nonumber
	\end{align}
	where the probability of the rare event is at most 
	\begin{align}
	& \exp\left(\frac{-2 \Delta_v^2 NK}{(64+16K) m^2U^4}\right) \nonumber \\
	& + \mathcal{N}(\mathcal{F}, \epsilon, d)\exp\left(\frac{-2 \rho^2 NK }{(64+16K) m^2U^4}\right) + \exp\left(\frac{-\gamma^2N}{32m^2U^4(1+1/K)}\right) \nonumber
	\end{align}
\end{theorem}
\begin{proof}
	The proof is in Appendix~\ref{proof:bound_residual}
\end{proof}

\subsection{Variance regularization}

\label{subsec:variance_regularization}

Bootstrapping diffusion involves applying variance regularization when training the residual denoiser $f_0(\cdot; \theta_0)$ by either penalizing the output variances or using a range adapter. The true motivation for variance regularization is to reduce the Lipschitz constant of residual denoiser.

As shown by the bound in Theorem~\ref{theorem:bound_residual}, the main component of the generalization error is determined by the covering numbers and Rademacher complexities. Additionally, the Rademacher complexity itself can be upper-bounded by covering numbers owing to Dudley's entropy integral. By reducing the Lipschitz constants, we effectively reduce the covering numbers and Radamacher complexities, which in turn results in lower generalization errors.

Ultimately, this improved generalization stems from the fact that training the residual diffusion denoiser becomes a simpler task when the combined denoiser already accurately estimates a significant portion of the score function. This less complex problem requires a network with reduced capacity. The difficulty in training the residual denoiser is proportional to the KL divergence $\left[\mu^A \,\|\, \mu(\theta_1, \ldots, \theta_I) \right]$. The KL divergence approximately quantifies the information missing from the partial data views. Consequently, the data required for training the residual denoiser is proportional to this uncaptured information. This suggests that bootstrapping diffusion only require data that are propositional to the difficulty of the problem. If the difficulty of the problem (the amount of uncaptured information) goes to zero, then the major components of the generation errors goes to zero. This implies that bootstrapping diffusion achieves near first-order optimal data efficiency.

\section{Experiments}

We perform experiments to validate our theoretical findings using the AFHQv2-Cat dataset (approximately 5,000 images) \citep{choi2020starganv2}. The full-resolution images are set to $256 \times 256$. All experiments are conducted in the latent space of a the Stable Diffusion v2 VAE. One partial view data we use is $8 \times 8$ patches and another partial view data are low-resolution down-sampled data. We present our generated images in Figure~\ref{fig:bootstrapping_vs_views}. The experiment results show that bootstrapping diffusion can effectively generate high-quality images with a small full-resolution training dataset. More experiment details are provided in Appendix~\ref{appendix:experiment}.

\section{Conclusion and Future Work}

We introduce bootstrapping diffusion, a novel method for training diffusion models in scenarios with scarce full-resolution data but abundant partial-view data. We posed and answered the theoretical question of whether partial-view data could be leveraged to achieve greater data efficiency when training full-resolution diffusion models. We demonstrated theoretically that our bootstrapping diffusion approach achieves near first-order optimal data efficiency. Additionally, our experimental results illustrate that bootstrapping diffusion effectively compensates for missing data, enabling successful training of full-resolution diffusion models. While this paper primarily emphasizes theoretical analysis, the experimental performance of the models can be further improved by optimizing aspects such as noise scheduling, learning rate scheduling, and network architecture design. The multiple diffusion models trained in bootstrapping diffusion can also be distilled into one diffusion model, which has not been investigated in this paper.

\bibliography{reference_paper}
\bibliographystyle{icml2023}

\newpage
\appendix
\onecolumn

\begin{center}{\Large{\textbf{Appendix}}}\end{center}

\section{Additional discussion on variance of the residual denoiser}

\label{appendix:residual_variance}

From another perspective if $\sum_{i=1}^I B_i f(A_i \mathbf{x}_t, t ; \theta_i)$ well approximates $\mathbb{E}[\mathbf{X}_{0} \mid A_1\mathbf{x}_t, \ldots, A_I\mathbf{x}_{t}]$, then $\sum_{i=1}^I B_i f(A_i \mathbf{x}_t, t ; \theta_i)$ is approximately a bias-free estimator of $\mathbb{E}[\mathbf{X}_{0} \mid \mathbf{x}_{t}]$. In other words, it follows from tower properties of expectations that
\begin{align}
\mathbb{E}\left[\mathbb{E}[\mathbf{X}_{0} \mid \mathbf{X}_t ] \middle| \sum_i B_i A_i\mathbf{X}_t \right] = \mathbb{E}\left[\mathbf{X}_{0} \middle| \sum_i B_iA_i\mathbf{X}_t \right]
\end{align}

In addition, Theorem~\ref{theorem:residual_variance} provides an identity that relating the MMSE (Minimum Mean Squared Error) $\mathbb{E}\left[\left(\mathbf{X}_0 - \mathbb{E}\left[\mathbf{X}_{0} \middle| \mathbf{X}_t \right]\right)^2\right]$, the error variance of the combined denoiser $\mathbb{E}\left[\left(\mathbf{X}_0 - \mathbb{E}\left[\mathbf{X}_0 \middle| \sum_i B_iA_i\mathbf{X}_t\right]\right)^2\right]$, and the variance of the residual denoiser $\mathbb{E}\left[\left(\mathbb{E}\left[\mathbf{X}_{0} \middle| \mathbf{X}_t \right] - \mathbb{E}\left[\mathbf{X}_0 \middle| \sum_i B_iA_i\mathbf{X}_t\right]\right)^2\right]$. It can be checked that if the error variance of combined denoiser approaches the MMSE, the residual denoiser should have small variances.

\begin{theorem}
	\label{theorem:residual_variance}
	\begin{align}
	& \mathbb{E}\left[\left(\mathbf{X}_0 - \mathbb{E}\left[\mathbf{X}_0 \middle| \sum_i B_iA_i\mathbf{X}_t\right]\right)^2\right] \nonumber \\
	& = \mathbb{E}\left[\left(\mathbf{X}_0 - \mathbb{E}\left[\mathbf{X}_{0} \middle| \mathbf{X}_t \right]\right)^2\right]  +  
	\mathbb{E}\left[\left(\mathbb{E}\left[\mathbf{X}_{0} \middle| \mathbf{X}_t \right] - \mathbb{E}\left[\mathbf{X}_0 \middle| \sum_i B_iA_i\mathbf{X}_t\right]\right)^2\right] \nonumber \\
	\end{align}
\end{theorem}
\begin{proof}
	The proof is in Appendix~\ref{proof:residual_variance}
\end{proof}

\section{More on experimental results}

\label{appendix:experiment}

We use the following partial-data views, patch views and low-resolution view. For patch views, we extract patches of size $8 \times 8$ from the original $32 \times 32$ resolution latents. Specifically, patches are taken from exactly 16 distinct locations. Each patch occupies rows from $8n$ to $8(n+1)-1$ and columns from $8m$ to $8(m+1)-1$, with $n,m \in \{0,1,2,3\}$. The projection matrix $A_i$ is defined as the operator that maps a full-resolution latent onto one of these patches. For low-resolution view, we downsample the original latents from the resolution of $32 \times 32$ to $8 \times 8$ using bilinear interpolation. The projection matrix $A_i$ is defined as the downsampling operator that maps the full-resolution latents to their corresponding low-resolution versions.

\paragraph{Partial view model training.} We train a single diffusion model for the 16 patch views without specifying the particular view. Specifically, the input to the patch-view diffusion model is a 
$8 \times 8$ patch with added noise, and the model predicts the corresponding expectations of the clean $8 \times 8$ patches. Similarly, we train another diffusion model specifically for the low-resolution view. In this case, the diffusion model takes as input the noisy $8 \times 8$ low-resolution latent and predicts the expectations of the corresponding clean low-resolution $8 \times 8$ latents. We train the low-resolution view diffusion model for 2600 epochs, observing that the training loss stops decreasing at this point. For the patch-view diffusion model, we train it for 1300 epochs.

\paragraph{Hyper-parameters at the range adapter.} We choose range adapter $s(\cdot; t)$ to be a piece-wise linear function with respect to the diffusion step $t$. We divide the diffusion step range $[0, T]$ into 100 intervals and estimate the $s(\cdot; t)$ value at the center of each interval. For determining these hyper-parameters, we train a standard diffusion model for the $32 \times 32$ full-resolution latents. We then generate 64 images using different random seeds and estimate the optimal $B_i$ and $s(\cdot; t)$ from the 64 sampling (reverse) diffusion processes.

\paragraph{Training dataset for the partial-view models.}
We create our training datasets based on the AFHQv2-Cat dataset. Specifically, we divide each latent representation from the original AFHQv2-Cat dataset into 16 individual patches. All these extracted patches are collected together, forming the training dataset for the patch-view diffusion model. Information regarding the original images corresponding to each patch, or associations between patches originating from the same image, is intentionally omitted. For low-resolution view diffusion models, we downsample the latent representations from the original AFHQv2-Cat dataset to obtain low-resolution latents, which serve as the training dataset for the low-resolution view diffusion models.

\paragraph{Training dataset for the residual bootstrapping diffusion model.} We randomly select 64 full-resolution images (latents) and make the 64 images as the only training data for the residual bootstrapping diffusion model.

\paragraph{Implementation details.}
Our implementation closely follows the sampling strategies, noise scheduling, and network architectures described in \citet{karras2022edm} and \citet{dhariwal2021diffusion}. Specifically, we adopt the preconditioning approach from \citet{karras2022edm}, and employ a U-Net architecture similar to that presented by \citet{dhariwal2021diffusion}.

\paragraph{Computing resources}
We train our models using a workstation with 2 Nvidia A100 GPUs.  

\paragraph{Generated images.}
Figure~\ref{fig:bootstrapping_vs_views} illustrates examples of generated images. The first row shows images produced by the low-resolution-view diffusion model. At each diffusion sampling step, the noisy latents are down-sampled to a resolution of $8 \times 8$. The model then applies $f(\cdot; \theta_i)$ to predict the expectation at this lower 
$8 \times 8$ resolution. Subsequently, we up-sample these predicted latent expectations back to the original resolution of $32 \times 32$ and use them for the next sampling iteration.

The second row of Figure~\ref{fig:bootstrapping_vs_views} displays images generated by the patch-view diffusion model. During each diffusion sampling step, the noisy latents are divided into 16 patches, and the model applies the function $f(\cdot; \theta_i)$ independently to each patch. The resulting denoised patches are then recombined into latent representations, which are utilized in the subsequent sampling iteration.
The third row of Figure~\ref{fig:bootstrapping_vs_views} presents images generated by linearly combining the outputs of the low-resolution-view and patch-view diffusion models at each diffusion sampling step.

The fourth row of Figure~\ref{fig:bootstrapping_vs_views} illustrates images produced using the bootstrapping diffusion approach described in Algorithm~\ref{algorithm:bootstrapping}. 

From the generated images, we observe that those produced by the low-resolution-view diffusion model typically lack high-frequency details. This is expected, since the low-resolution data inherently contain limited information regarding high-frequency components. Conversely, images generated by the patch-view diffusion model exhibit rich local details but generally lack coherent high-level structure, as patch-view data do not include global structural information. Additionally, images generated through a straightforward linear combination approach fail to produce accurate global structures. We attribute this to the absence of high-frequency global structural information in both the patch-view and low-resolution-view datasets. Finally, the bootstrapping diffusion method effectively compensates for this missing information, resulting in high-quality, well-structured generated images.

\begin{figure*}[!ht]
	\centering
	\includegraphics[width=\textwidth]{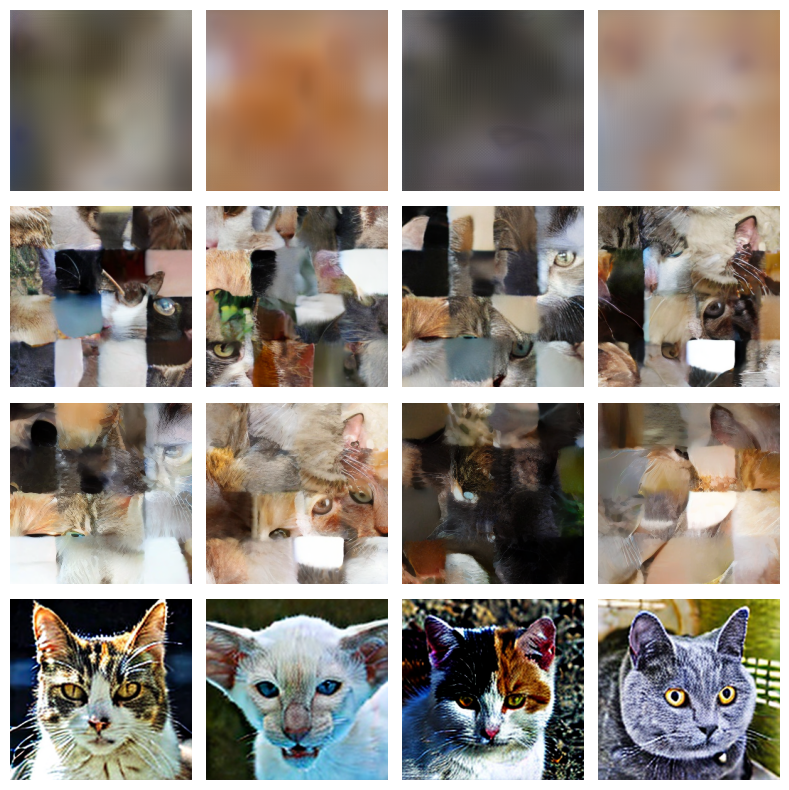}
	\caption{Randomly generated images. The first row of images are generated using the low-resolution view diffusion model. The second row of images are generated using the patch view diffusion model. The third row of images are generated using a combination of low-resolution view and patch view diffusion models. The fourth row of images are generated using the bootstrapping diffusion model proposed in this paper. }
	\label{fig:bootstrapping_vs_views}
\end{figure*}

\section{Potential Social Implications}

Our work may lead to both positive and negative potential societal impacts. Specifically, we identify positive impacts such as improvements in computational efficiency, reduced resource usage, and enabling broader accessibility of advanced image generation techniques for beneficial applications (e.g., artistic creativity, educational content, data augmentation for medical imaging). At the same time, we acknowledge potential negative impacts, including the misuse of generated images in creating misleading or deceptive content, such as fake news, misinformation campaigns, or unauthorized impersonation.

\section{Auxiliary lemmas used in the proof}

First, let us recall the Azuma-Hoeffding inequality as presented in Theorem 13.4 of \citet{mitzenmacher2017probability}, originally due to \citet{azuma1967weighted}.
\begin{theorem}[Generalized Azuma's inequality]
	Let $(X_k)_{k=0}^n$ be a martingale adapted to filtration $\{\mathcal{F}_k\}$. Suppose there exist constants $a_k,b_k$ such that
	\[
	a_k \leq X_k - X_{k-1} \leq b_k,\quad \text{for all } 1 \leq k \leq n.
	\]
	Then, for any $t>0$, we have
	\[
	P(X_n - X_0 \geq t) \leq \exp\left(-\frac{2t^2}{\sum_{k=1}^{n}(b_k - a_k)^2}\right).
	\]
	Similarly,
	\[
	P(X_n - X_0 \leq -t) \leq \exp\left(-\frac{2t^2}{\sum_{k=1}^{n}(b_k - a_k)^2}\right).
	\]
\end{theorem}

We will use the tower properties of expectations as presented in Theorem 34.4 of \citet{billingsley1995probability}
\begin{theorem}
	If a random variable $X$ is integrable and the $\sigma$-algebras $\mathcal{F}_1$ and $\mathcal{F}_2$ satisfy $\mathcal{F}_1 \subset \mathcal{F}_2$, then
	\begin{align}
	\mathbb{E}\left[\mathbb{E}[X \mid \mathcal{F}_2] \mid \mathcal{F}_1\right] = \mathbb{E}[X \mid \mathcal{F}_1] 
	\end{align}
\end{theorem}

For any $\theta$, we can always rewrite $\mathcal{L}(\theta)$ as follows,
\begin{align}
\label{eq:loss_alternative}
\mathcal{L}(\theta) & = \sum_{k=1}^{K} \sum_{n=1}^{N} \left\| f(\mathbf{x}_{n,k}, t_{n,k}; \theta) - \mathbf{x}_{n,0} \right\|_2^2 \nonumber \\
& = \sum_{k=1}^{K} \sum_{n=1}^{N} \left\| f(\mathbf{x}_{n,k}, t_{n,k}; \theta) - \mathbb{E}[\mathbf{X}_{n,0} \mid \mathbf{x}_{n,k}, t_{n,k}] + \mathbb{E}[\mathbf{X}_0 \mid \mathbf{x}_{n,k}, t_{n,k}] - \mathbf{x}_{n,0} \right\|_2^2 \nonumber \\
& = \sum_{k=1}^{K} \sum_{n=1}^{N} \left\| f(\mathbf{x}_{n,k}, t_{n,k}; \theta) - E[\mathbf{X}_{n,k} \mid \mathbf{x}_{n,k}, t_{n,k}]\right\|^2 + \left\|E[\mathbf{X}_{n,k} \mid \mathbf{x}_{n,k}, t_{n,k}] - \mathbf{x}_{n,0} \right\|_2^2 
\nonumber  \\
& \hspace{0.5in} + 2 (f(\mathbf{x}_{n,k}, t_{n,k}; \theta) - \mathbb{E}[\mathbf{X}_{n,0} \mid \mathbf{x}_{n,k}, t_{n,k}])^T (\mathbb{E}[\mathbf{X}_{n,0} \mid \mathbf{x}_{n,k}, t_{n,k}] - \mathbf{x}_{n,0} ) 
\end{align}

\begin{lemma}
	For any $\theta\in\Theta$,
	\begin{align}
	\label{eq:expect_loss_aux1}
	\mathbb{E} \left[\sum_{k=1}^{K} \left(f(\mathbf{X}_{n,k}, t_{n,k}; \theta) - E[\mathbf{X}_{n,0} \mid \mathbf{X}_{n,k}]\right)^T \left(\mathbb{E}[\mathbf{X}_{n,0} \mid \mathbf{X}_{n,k}] - \mathbf{X}_{n,0} \right) \right] = 0 
	\end{align}
\end{lemma}
\begin{proof}
	\begin{align}
	& \mathbb{E} \left[\sum_{k=1}^{K} \left(f(\mathbf{X}_{n,k}, t_{n,k}; \theta) - \mathbb{E}[\mathbf{X}_{n,0} \mid \mathbf{X}_{n,k}]\right)^T \left(\mathbb{E}[\mathbf{X}_{n,0} \mid \mathbf{X}_{n,k}] - \mathbf{X}_{n,0} \right) \right] \nonumber \\
	\hspace{0.5in } & \stackrel{(a)}{=} \mathbb{E} \left[\mathbb{E} \left[\sum_{k=1}^{K} \left(f(\mathbf{X}_{n,k}, t_{n,k}; \theta) - \mathbb{E}[\mathbf{X}_{n,0} \mid \mathbf{X}_{n,k}]\right)^T \left(\mathbb{E}[\mathbf{X}_{n,0} \mid \mathbf{X}_{n,k}] - \mathbf{X}_{n,0} \right) \right] \middle| \mathbf{X}_{n, k}, t_{n, k} \right] \nonumber \\
	\hspace{0.5in } & \stackrel{(b)}{=} \mathbb{E} \left[\sum_{k=1}^{K} \left(f(\mathbf{X}_{n,k}, t_{n,k}; \theta) - \mathbb{E}[\mathbf{X}_{n,0} \mid \mathbf{X}_{n,k}]\right)^T \left(\mathbb{E}[\mathbf{X}_{n,0} \mid \mathbf{X}_{n,k}] - \mathbb{E}[\mathbf{X}_{n,0} \mid \mathbf{X}_{n, k}, t_{n,k}] \right) \right] = 0 \nonumber
	\end{align}
	where (a) follows from the tower property of expectations, and (b) follows from moving everything $\mathbf{X}_{n,k}, t_{n,k}$ measurable out of the expectation $E[\cdot \mid x_{i,k}, t_{i,k}]$
\end{proof}

In order to prove probability upper bound using Azuma's inequation, we need frequently the following filtration.  Let us use $\Omega$ to denote the sample space. We define a sequence of $\sigma$-algebras $\mathcal{A}_0, \mathcal{A}_{1,0}, \mathcal{A}_{1,1}, \mathcal{A}_{1,2}, \ldots, \mathcal{A}_{1,K}, \mathcal{A}_{2,0}, \mathcal{A}_{2,1}, \ldots, \mathcal{A}_{2,K}, \ldots, \mathcal{A}_{ik}, \ldots$ as follows.
\begin{align}
& \mathcal{A}_{0,0} = \{\Omega, \emptyset\} \nonumber \\
& \mathcal{A}_{n,0} = \sigma(\mathbf{X}_{1,0}, t_{1,1}, \mathbf{X}_{1,1}, t_{1,2}, \mathbf{X}_{1,2}, \ldots, t_{1,K}, \mathbf{X}_{1,K}, \mathbf{X}_{2,0}, t_{2,1}, \mathbf{X}_{2,1}, \ldots, \mathbf{X}_{n,0}) \nonumber \\
& \mathcal{A}_{n,k} = \sigma(
\mathbf{X}_{1,0}, t_{1,1}, \mathbf{X}_{1,1}, t_{1,2}, \mathbf{X}_{1,2}, \ldots, t_{1,K},\mathbf{X}_{1,K}, \mathbf{X}_{2,0}, t_{2,1}, \mathbf{X}_{2,1}, \ldots, \mathbf{X}_{n,0}, \ldots, t_{n,k}, \mathbf{X}_{n,k}) \nonumber
\end{align}
In other words, $\mathcal{A}_0$ is the $\sigma$-algebra that no random variables are observed, $\mathcal{A}_{n,0}$ is the $\sigma$-algebra generated by random variables $\mathbf{X}_{1,0}$, $t_{1,1}$, $\mathbf{X}_{1,1}$, $t_{1,2}$, $\mathbf{X}_{1,2}$, $\ldots$, $t_{1,K}$,  $\mathbf{X}_{1,K}$, $\mathbf{X}_{2,0}$, $t_{2,1}$, $\mathbf{X}_{2,1}$, $\ldots$, $\mathbf{X}_{n,0}$, (the first $n$ data samples in the dataset $S_0$ and the noisy samples $t_{n,k}$, $\mathbf{X}_{n,k}$ corresponding to the first $n-1$ data samples in $S_0$). The $\sigma$-algebra $\mathcal{A}_{n,k}$ is generated by random variables $\mathbf{X}_{1,0}$, $t_{1,1}$, $\mathbf{X}_{1,1}$, $t_{1,2}$, $\mathbf{X}_{1,2}$, $\ldots$, $t_{1,K}$, $\mathbf{X}_{1,K}$, $\mathbf{X}_{2,0}$, $t_{2,1}$, $\mathbf{X}_{2,1}$, $\ldots$, $\mathbf{X}_{n,0}$, $\ldots$, $t_{n,k}$, $\mathbf{X}_{n,k}$, (compared with $\mathcal{A}_{n,0}$, random variables $t_{n,1}$, $\mathbf{X}_{n,1}$, $\ldots$, $t_{n,k}$ $\mathbf{X}_{n,k}$ are additionally observed). It can be checked that the sequence of $\sigma$-algebras $\mathcal{A}_0$, $\mathcal{A}_{1,0}$, $\mathcal{A}_{1,1}$, $\mathcal{A}_{1,2}$, $\ldots$, $\mathcal{A}_{1,K}$, $\mathcal{A}_{2,0}$, $\mathcal{A}_{2,1}$, $\ldots$, $\mathcal{A}_{2,K}$, $\ldots$, $\mathcal{A}_{n,k}$, $\ldots$ is a filtration. 

The two loss functions are related as follows,
\begin{lemma}
	\label{lemma:losses}
	\begin{align}
	L(\theta) = R(\theta) + \mathbb{E}\left[\mathcal{V}(\mathcal{S}_0)\right]
	\end{align}
\end{lemma}
\begin{proof}
Note the identity in Equation~\ref{eq:loss_alternative}. Then the lemma follows from Equation~\ref{eq:expect_loss_aux1}.
\end{proof}

\section{Proof of theorem~\ref{theorem:bound_e1}}
\label{proof:bound_e1}
\begin{theorem}
The probability $\Pr(\mathcal{E}_1)$ is upper bounded by
\begin{align}
\Pr(\mathcal{E}_1) \leq \exp\left(\frac{-2 \Delta_v^2 NK}{ (64+16K) m^2U^4}\right)
\end{align}
\end{theorem}
\begin{proof}

Note $\mathcal{E}_1$ denote the random event that 
\begin{align}
\mathcal{L}(\theta^\ast) \geq \left(\mathbb{E}[\mathcal{V}(\mathcal{S}_0)] + \Delta_b^2 + \Delta_v^2\right) NK
\end{align}

Consider the following sequence of random variables $\{\ldots, \mathbb{E}\left[\mathcal{L}(\theta^\ast) \mid \mathcal{A}_{n,k}\right], \ldots\}$, 
It can be checked that this sequence of random variables is a martingale (in fact, it is called the Doob's martingale \citet{doob1953stochastic} and also see Example 4.1.4 in \citet{welsh2023little}). Let $\mathbf{Z}_{n,k}$ denote the martingale incremental, for $k=1,2,\ldots, K$
\begin{align}
\mathbf{Z}_{n,k} = \mathbb{E}\left[\mathcal{L}(\theta^\ast) \mid \mathcal{A}_{n,k}\right] - \mathbb{E}\left[\mathcal{L}(\theta^\ast) \mid \mathcal{A}_{n,k-1}\right]
\end{align}
\begin{align}
\mathbf{Z}_{n,0} = \mathbb{E}\left[\mathcal{L}(\theta^\ast) \mid \mathcal{A}_{n,0}\right] - \mathbb{E}\left[\mathcal{L}(\theta^\ast) \mid \mathcal{A}_{n-1,K}\right]
\end{align}
It can be checked that for $k=1,\ldots,K$
\begin{align}
\mathbf{Z}_{n,k} =  \left\| f(\mathbf{x}_{n,k}, t_{n,k}; \theta) - \mathbf{x}_{n,0} \right\|_2^2 - \mathbb{E}\left[\left\| f(\mathbf{X}_{n,k}, t_{n,k}; \theta) - \mathbf{x}_{n,0} \right\|_2^2 \middle| \mathbf{X}_{n,k}\right]
\end{align}
\begin{align}
\mathbf{Z}_{n,0} & =  \sum_{k=1}^{K}\mathbb{E}\left[\left\| f(\mathbf{X}_{n,k}, t_{n,k}; \theta) - \mathbf{x}_{n,0} \right\|_2^2 \middle| \mathbf{x}_{n,0}\right] - \sum_{k=1}^{K}\mathbb{E}\left[\left\| f(\mathbf{X}_{n,k}, t_{n,k}; \theta) - \mathbf{X}_{n,0} \right\|_2^2\right] \nonumber \\
& = \sum_{k=1}^{K}\mathbb{E}\left[\left\| f(\mathbf{X}_{n,k}, t_{n,k}; \theta) - \mathbf{x}_{n,0} \right\|_2^2 \middle| \mathbf{x}_{n,0}\right] - K(\mathbb{E}[\mathcal{V}(\mathcal{S}_0)] + \Delta_b^2)
\end{align}
Because all $\mathbf{X}_{n,k}$ are element-wise bounded between $[-U, U]$ and the dimension is $m$,
\begin{align}
-4mU^2 \leq \mathbf{Z}_{n,k} \leq 4mU^2
\end{align}
\begin{align}
- K(\mathbb{E}[\mathcal{V}(\mathcal{S}_0)] + \Delta_b^2) \leq \mathbf{Z}_{n,0} \leq 4KmU^2 - K(\mathbb{E}[\mathcal{V}(\mathcal{S}_0)] + \Delta_b^2)
\end{align}

Also note that
\begin{align}
\mathbb{E}\left[\mathcal{L}(\theta^\ast)\right] = \left(\mathbb{E}[\mathcal{V}(\mathcal{S}_0)] + \Delta_b^2\right) NK
\end{align}

Then by Azuma's inequality, 
\begin{align}
\Pr(\mathcal{E}_1) \leq \exp\left(\frac{-2 \Delta_v^2 NK}{ (64+16K) m^2U^4}\right)
\end{align}
\end{proof}

\section{Proof of Theorem~\ref{theorem:bound_e2}}
\label{proof:bound_e2}
\begin{theorem}
	The probability $\Pr(\mathcal{E}_2)$ is upper bounded by
\begin{align}
\Pr[\mathcal{E}_2] \leq \mathcal{N}(\mathcal{F}, \epsilon, d)\exp\left(\frac{-2 \rho^2 NK }{ (64+16K) m^2U^4}\right)
\end{align}
\end{theorem}
\begin{proof}

Recall that $\mathcal{E}_2$ denote the random event that for one of the $\widehat{\theta}\in \mathcal{C}_1$, 
\begin{align}
\mathcal{L}(\widehat{\theta})  \leq \left( \sqrt{\mathbb{E}[\mathcal{V}(\mathcal{S}_0)] + \Delta_b^2 + \Delta_v^2} + \epsilon\right)^2 NK
\end{align}

For each particular $\widehat{\theta}\in \mathcal{C}_1$, by definition
\begin{align}
\mathbb{E}\left[{\mathcal{L}(\widehat{\theta})} \right]
& =  \left(R(\widehat{\theta}) + \mathbb{E}[\mathcal{V}(\mathcal{S}_0)]\right) NK \geq  \left( \left(\sqrt{\mathbb{E}[\mathcal{V}(\mathcal{S}_0)] + \Delta_b^2 + \Delta_v^2} + \epsilon\right)^2 + \rho \right) NK \nonumber 
\end{align}
Following almost the same martingale argument in Appendix~\ref{proof:bound_e1}, we can show that 
\begin{align}
\Pr\left[\frac{\mathcal{L}(\widehat{\theta})}{NK}  \leq \left( \sqrt{\mathbb{E}[\mathcal{V}(\mathcal{S}_0)] + \Delta_b^2 + \Delta_v^2} + \epsilon\right)^2\right] \leq \exp\left(\frac{-2 \rho^2 NK}{ (64+16K) m^2U^4}\right)
\end{align}
We have at most $\mathcal{N}(\mathcal{F}, \epsilon, d)$ elements in $\mathcal{C}_1$, by using a union bound, we can prove that
\begin{align}
\Pr[\mathcal{E}_2] \leq \mathcal{N}(\mathcal{F}, \epsilon, d)\exp\left(\frac{-2 \rho^2 NK }{ (64+16K) m^2U^4}\right)
\end{align}
\end{proof}

\section{Proof of Theorem \ref{theorem:bound_radamacher}}
\label{proof:bound_radamacher}

\begin{theorem}
	\label{theorem:bound_radamacher}
	\begin{align}
	\Pr\left[\sup_{\theta} \left| L(\theta) - \widehat{L}(\theta) \right | \geq  2\mathfrak{R}_L(\mathcal{F}, \mathcal{S}_0) + \gamma \right] \leq \exp\left(\frac{-\gamma^2N}{32m^2U^4(1+1/K)}\right)
	\end{align}
	\begin{align}
	\Pr\left[\sup_{\theta} \left| R(\theta) - \widehat{R}(\theta) \right | \geq 2\mathfrak{R}_R(\mathcal{F}, \mathcal{S}_0) + \gamma \right] \leq \exp\left(\frac{-\gamma^2N}{32m^2U^4(1+1/K)}\right)
	\end{align}
\end{theorem}

\begin{proof}

With the standard symmetrization argument, we first introduce a ghost sample set $\mathcal{S}'_0 = \{\cdots, x'_{n,k}, \ldots \}\}$ . Having now two sample sets $\mathcal{S}_0$ and $\mathcal{S}'_0$, let us denote  
\begin{align}
\ell(\mathbf{x}_{n,k}, t_{n,k}; \theta) =  \left\|f(\mathbf{x}_{n,k}, t_{n,k}; \theta) - \mathbf{x}_{ n,0}\right\|_2^2
\end{align}
\begin{align}
\ell(\mathbf{x}'_{n,k}, t'_{n,k}; \theta) =  \left\|f(\mathbf{x}'_{n,k}, t'_{n,k}; \theta) - \mathbf{x}'_{ n,0} \right\|_2^2
\end{align}
\begin{align}
\widehat{L}(\theta)(\mathcal{S}_0) =  \sum_n \sum_k \frac{1}{NK}\left\|f(\mathbf{x}_{n,k}, t_{n,k}; \theta) - \mathbf{x}_{ n,0} \right\|_2^2
\end{align}
\begin{align}
\widehat{L}(\theta)(\mathcal{S}'_0) =  \sum_n \sum_k \frac{1}{NK}\left\|f(\mathbf{x}'_{n,k}, t'_{n,k}; \theta) - \mathbf{x}'_{ n,0} \right\|_2^2
\end{align}
\begin{align}
L(\theta) =  \mathbb{E}\left[\left\|f(\mathbf{x}_{n,k}, t_{n,k}; \theta) - \mathbf{X}_{ n,0} \right\|_2^2 \right]
\end{align}

then we can bound
\begin{align}
& \mathbb{E}_{\mathcal{S}_0} \left[ \sup_{\theta \in \Theta} \left( L(\theta) - \hat{L}(\theta)(\mathcal{S}_0) \right) \right] \nonumber \\
&  = \mathbb{E}_{\mathcal{S}_0} \left[ \sup_{\theta \in \Theta} \left( \mathbb{E}_{\mathcal{S}_0}[\ell(\mathbf{X}_{n,k}, t_{n,k};\theta)] - \frac{1}{NK} \sum_{n=1}^{N}\sum_{k=1}^{K} \ell(\mathbf{x}_{n,k}, t_{n,k}; \theta) \right) \right] \nonumber \\
&  \stackrel{(a)}{=} \mathbb{E}_{\mathcal{S}_0} \left[ \sup_{\theta \in \Theta} \left( \mathbb{E}_{\mathcal{S}'_0}[\ell'(\mathbf{X}'_{n,k}, t'_{n,k}; \theta)] - \frac{1}{NK} \sum_{n=1}^{N}\sum_{k=1}^{K} \ell(\mathbf{x}_{n,k}, t_{n,k}; \theta) \right) \right] \nonumber \\
&  \stackrel{(b)}{\leq} \mathbb{E}_{\mathcal{S}_0} \left[ \mathbb{E}_{\mathcal{S}'_0} \left[ \sup_{\theta \in \Theta} \left( \ell'(\mathbf{X}'_{n,k}, t'_{n,k}; \theta) - \frac{1}{NK} \sum_{n=1}^{N}\sum_{k=1}^{K} \ell(\mathbf{x}_{n,k}, t_{n,k}; \theta) \right) \right] \right] \nonumber \\
&  = \mathbb{E}_{\mathcal{S}_0} \left[ \mathbb{E}_{\mathcal{S}'_0} \left[ \sup_{\theta \in \Theta} \frac{1}{NK}\left( \sum_{n=1}^{N}\sum_{k=1}^{K} \ell'(\mathbf{x}'_{n,0}, t'_{n,k}; \theta) - \sum_{n=1}^{N}\sum_{k=1}^{K} \ell(\mathbf{x}_{n,k}, t_{n,k}; \theta) \right) \right] \right] \nonumber \\ 
&  \stackrel{(c)}{\leq} \mathbb{E}_{\mathcal{S}_0} \left[ \mathbb{E}_{\mathcal{S}'_0} \left[ \sup_{\theta \in \Theta} \frac{1}{NK} \sum_{n=1}^{N}\sum_{k=1}^{K} \sigma_{n,k} \left(\ell'(\mathbf{x}'_{n,0}, t'_{n,k}; \theta) - \ell(\mathbf{x}_{n,k}, t_{n,k}; \theta)\right)  \right] \right] \nonumber \\ 
&  \stackrel{(d)}{\leq} 2 \mathbb{E}_{\mathcal{S}_0} \left[ \sup_{\theta \in \Theta} \frac{1}{NK} \sum_{n=1}^{N}\sum_{k=1}^{K} \sigma_{n,k}  \ell(\mathbf{x}_{n,k}, t_{n,k}; \theta)  \right] = 2 \, \mathfrak{R}_L(\mathcal{F}, \mathcal{S}_0).\nonumber 
\end{align}
In the above inequalities, 
(a) follows from the fact that $\mathcal{S}_0$ and $\mathcal{S}'_0$ have the same distribution, (b) follows from the Jensen's inequality (pulling the supremum inside the expectation), (c) follows from that the fact that the two sides are equal if $\sigma_{n,k} = 1$ (d) follows from symmetry of the two random datasets $\mathcal{S}_0$ and $\mathcal{S}'_0$.

Define again a martingale 
$ \mathbb{E} \left[ \sup_{\theta \in \Theta} \left( L(f) - \hat{L}(\mathcal{S}_0) \right)  \middle| \mathcal{A}_{n,k} \right]$. The martingale increments
\begin{align}
\mathbf{Z}_{n,0} & = \mathbb{E} \left[ \sup_{\theta \in \Theta} \left( L(f) - \hat{L}(\mathcal{S}_0) \right)  \middle| \mathcal{A}_{n,0} \right] - \mathbb{E} \left[ \sup_{\theta \in \Theta} \left( L(f) - \hat{L}(\mathcal{S}_0) \right)  \middle| \mathcal{A}_{n-1,K} \right] \nonumber \\
\end{align}
Conditioned on an additional realization $\mathbf{x}_{n,0}$ would impact at most $K$ terms $\ell(\mathbf{x}_{n,k}, t_{n,k}; \theta)$. Thus
\begin{align}
|\mathbf{Z}_{n,0} | \leq \frac{4mU^2}{N}
\end{align}

With a similar argument, for $k=1,\ldots, K$, 
\begin{align}
\mathbf{Z}_{n,k} & = \mathbb{E} \left[ \sup_{\theta \in \Theta} \left( L(f) - \hat{L}(\mathcal{S}_0) \right)  \middle| \mathcal{A}_{n,k} \right] - \mathbb{E} \left[ \sup_{\theta \in \Theta} \left( L(f) - \hat{L}(\mathcal{S}_0) \right)  \middle| \mathcal{A}_{n,k-1} \right] \nonumber \\
\end{align}
\begin{align}
|\mathbf{Z}_{n,k} | \leq \frac{4mU^2}{NK}
\end{align}

Using Azuma's inequality,
\begin{align}
\Pr\left[\sup_{\theta} \left| L(\theta) - \widehat{L}(\theta) \right | \geq 2\mathfrak{R}_L(\mathcal{F}, \mathcal{S}_0) + \gamma \right] \leq \exp\left(\frac{-\gamma^2N}{32m^2U^4(1+1/K)}\right)
\end{align}

The bound for 
$\Pr\left[\sup_{\theta} \left| R(\theta) - \widehat{R}(\theta) \right | \geq 2\mathfrak{R}_R(\mathcal{F}, \mathcal{S}_0) + \gamma \right] $ can be proved similarly.

\end{proof}

\section{Proof of Theorem \ref{theorem:bound_full_resolution}}
\label{proof:bound_full_resolution}

\begin{theorem}
Unless a rare event occurs, the minimizer of the loss function $\mathcal{L}(\theta)$ achieves a generalization error
\begin{align}
R(\theta_2) \leq & \left(\sqrt{\left( \left(\sqrt{\mathbb{E}[\mathcal{V}(\mathcal{S}_0)] + \Delta_b^2 + \Delta_v^2} + \epsilon\right)^2 + \rho + 2 \mathfrak{R}_L(\mathcal{F}, \mathcal{S}_0) + \gamma \right)} + \epsilon \right)^2 \nonumber \\
& + 2 \mathfrak{R}_L(\mathcal{F}, \mathcal{S}_0) + \gamma - \mathbb{E}[\mathcal{V}(\mathcal{S}_0)] \nonumber
\end{align}
where the probability of the rare event is at most 
\begin{align}
& \exp\left(\frac{-2 \Delta_v^2 NK}{(64+16K) m^2U^4}\right) \nonumber \\
& + \mathcal{N}(\mathcal{F}, \epsilon, d)\exp\left(\frac{-2 \rho^2 NK }{(64+16K) m^2U^4}\right) + \exp\left(\frac{-\gamma^2N}{32m^2U^4(1+1/K)}\right) \nonumber
\end{align}
\end{theorem}
\begin{proof}
	Suppose events $\mathcal{E}_1$, $\mathcal{E}_2$ and $\mathcal{E}_3$ have not happened. Then for all $\widehat{\theta} \in \mathcal{C}_1$, 
	\begin{align}
	\mathcal{L}(\widehat{\theta})  \geq \left( \sqrt{\mathbb{E}[\mathcal{V}(\mathcal{S}_0)] + \Delta_b^2 + \Delta_v^2} + \epsilon\right)^2 NK
	\end{align}
	Then for any $\theta$ contained in the $\epsilon$-balls with centers in $\mathcal{C}_1$, by triangle inequality,
	\begin{align}
	\mathcal{L}(\theta)  \geq  \left(\mathbb{E}[\mathcal{V}(\mathcal{S}_0)] + \Delta_b^2 + \Delta_v^2 \right) NK
	\end{align}
	Because random event $\mathcal{E}_1$ does not occur 
	\begin{align}
	\mathcal{L}(\theta^\ast) \leq \left(\mathbb{E}[\mathcal{V}(\mathcal{S}_0)] + \Delta_b^2 + \Delta_v^2 \right) NK
	\end{align}
	The above implies that for all $\theta$ contained in the $\epsilon$-balls with centers in $\mathcal{C}_1$, they can not be the global minimizer of $\mathcal{L}(\theta)$. 
	
	Next, consider all $\theta_1 \in \mathcal{C}_2$, by the definition, we have
	\begin{align}
	R(\theta_1) \leq \left(\sqrt{\mathbb{E}[\mathcal{V}(\mathcal{S}_0)] + \Delta_b^2 + \Delta_v^2} + \epsilon\right)^2 - \mathbb{E}[\mathcal{V}(\mathcal{S}_0]+ \rho
	\end{align}
	\begin{align}
	L(\theta_1) \leq \left(\sqrt{\mathbb{E}[\mathcal{V}(\mathcal{S}_0)] + \Delta_b^2 + \Delta_v^2} + \epsilon\right)^2 + \rho
	\end{align}
	Because random event $\mathcal{E}_3$ does not occur, and also from Theorem~\ref{theorem:bound_radamacher}, we have for each $\theta_1$
	\begin{align}
	\widehat{L}(\theta_1) \leq \left(\sqrt{\mathbb{E}[\mathcal{V}(\mathcal{S}_0)] + \Delta_b^2 + \Delta_v^2} + \epsilon\right)^2 + \rho + 2 \mathfrak{R}_L(\mathcal{F}, \mathcal{S}_0) + \gamma
	\end{align}
	
	Now consider $\theta_2$ that are contained in one of the $\epsilon$-ball with centers in $\mathcal{C}_2$, by definition and triangle inequality
	\begin{align}
	\widehat{L}(\theta_2) \leq \left(\sqrt{\left( \left(\sqrt{\mathbb{E}[\mathcal{V}(\mathcal{S}_0)] + \Delta_b^2 + \Delta_v^2} + \epsilon\right)^2 + \rho + 2 \mathfrak{R}_L(\mathcal{F}, \mathcal{S}_0) + \gamma \right)} + \epsilon \right)^2
	\end{align}
	Again, because $\mathcal{E}_3$ does not occur,
	\begin{align}
	L(\theta_2) & \leq \left(\sqrt{\left( \left(\sqrt{\mathbb{E}[\mathcal{V}(\mathcal{S}_0)] + \Delta_b^2 + \Delta_v^2} + \epsilon\right)^2 + \rho + 2 \mathfrak{R}_L(\mathcal{F}, \mathcal{S}_0) + \gamma \right)} + \epsilon \right)^2 \nonumber \\
	 & + 2 \mathfrak{R}_L(\mathcal{F}, \mathcal{S}_0) + \gamma
	\end{align}	
	\begin{align}
	R(\theta_2) \leq & \left(\sqrt{\left( \left(\sqrt{\mathbb{E}[\mathcal{V}(\mathcal{S}_0)] + \Delta_b^2 + \Delta_v^2} + \epsilon\right)^2 + \rho + 2 \mathfrak{R}_L(\mathcal{F}, \mathcal{S}_0) + \gamma \right)} + \epsilon \right)^2 \nonumber \\
	& + 2 \mathfrak{R}_L(\mathcal{F}, \mathcal{S}_0) + \gamma - \mathbb{E}[\mathcal{V}(\mathcal{S}_0)]
	\end{align}
	
By a union bound, the probability that one of the events $\mathcal{E}_1$, $\mathcal{E}_2$, $\mathcal{E}_3$ occurs is at most
\begin{align}
& \exp\left(\frac{-2 \Delta_v^2 NK}{(64+16K) m^2U^4}\right) \nonumber \\
& + \mathcal{N}(\mathcal{F}, \epsilon, d)\exp\left(\frac{-2 \rho^2 NK }{(64+16K) m^2U^4}\right) + \exp\left(\frac{-\gamma^2N}{32m^2U^4(1+1/K)}\right) \nonumber
\end{align}
	
\end{proof}

\section{Proof of Theorem~\ref{theorem:residual_variance}}

\label{proof:residual_variance}

Restate the theorem
\begin{theorem}
	\begin{align}
	& \mathbb{E}\left[\left(\mathbf{X}_0 - \mathbb{E}\left[\mathbf{X}_0 \middle| \sum_i B_iA_i\mathbf{X}_t\right]\right)^2\right] \nonumber \\
	& = \mathbb{E}\left[\left(\mathbf{X}_0 - \mathbb{E}\left[\mathbf{X}_{0} \middle| \mathbf{X}_t \right]\right)^2\right]  +  
	\mathbb{E}\left[\left(\mathbb{E}\left[\mathbf{X}_{0} \middle| \mathbf{X}_t \right] - \mathbb{E}\left[\mathbf{X}_0 \middle| \sum_i B_iA_i\mathbf{X}_t\right]\right)^2\right] \nonumber \\
	\end{align}
\end{theorem}
\begin{proof}
	\begin{align}
	& \mathbb{E}\left[\left(\mathbf{X}_0 - \mathbb{E}\left[\mathbf{X}_0 \middle| \sum_i B_iA_i\mathbf{X}_t\right]\right)^2\right] \nonumber \\
	& = \mathbb{E}\left[\left(\mathbf{X}_0 - \mathbb{E}\left[\mathbf{X}_{0} \mid \mathbf{X}_t \right]  + \mathbb{E}[\mathbf{X}_0 \mid \mathbf{X}_t]-  \mathbb{E}\left[\mathbf{X}_{0} \middle| \sum_i B_iA_i\mathbf{X}_t \right] \right)^2\right] \nonumber \\
	& = \mathbb{E}\left[\left(\mathbf{X}_0 - \mathbb{E}\left[\mathbf{X}_{0} \middle| \mathbf{X}_t \right]\right)^2\right]  +  
	\mathbb{E}\left[\left(\mathbb{E}\left[\mathbf{X}_{0} \middle| \mathbf{X}_t \right] - \mathbb{E}\left[\mathbf{X}_0 \middle| \sum_i B_iA_i\mathbf{X}_t\right]\right)^2\right] \nonumber \\
	& + 2\mathbb{E}\left[\left(\mathbf{X}_0 - \mathbb{E}\left[\mathbf{X}_{0} \middle| \mathbf{X}_t \right]\right)^T\left(\mathbb{E}\left[\mathbf{X}_{0} \middle| \mathbf{X}_t \right] - \mathbb{E}\left[\mathbf{X}_0 \middle| \sum_i B_iA_i\mathbf{X}_t\right]\right)\right] \nonumber \\
	\end{align}
	The lemma then follows from
	\begin{align}
	& \mathbb{E}\left[\left(\mathbf{X}_0 - \mathbb{E}\left[\mathbf{X}_{0} \middle| \mathbf{X}_t \right]\right)^T\left(\mathbb{E}\left[\mathbf{X}_{0} \middle| \mathbf{X}_t \right] - \mathbb{E}\left[\mathbf{X}_0 \middle| \sum_i B_iA_i\mathbf{X}_t\right]\right)\right] \nonumber \\
	& \stackrel{(a)}{=} \mathbb{E}\left[\left(\mathbf{X}_0 - \mathbb{E}\left[\mathbf{X}_{0} \middle| \mathbf{X}_t \right]\right)^T\left(\mathbb{E}\left[\mathbf{X}_{0} \middle| \mathbf{X}_t \right] - \mathbb{E}\left[\mathbf{X}_0 \middle| \sum_i B_iA_i\mathbf{X}_t\right]\right) \middle| \mathbf{X}_t \right] \nonumber \\
	& \stackrel{(b)}{=} \mathbb{E}\left[\left(\mathbb{E}[\mathbf{X}_0 \mid \mathbf{X}_t] - \mathbb{E}\left[\mathbf{X}_{0} \middle| \mathbf{X}_t \right]\right)^T\left(\mathbb{E}\left[\mathbf{X}_{0} \middle| \mathbf{X}_t \right] - \mathbb{E}\left[\mathbf{X}_0 \middle| \sum_i B_iA_i\mathbf{X}_t\right]\right) \right] =0
	\end{align}
	where (a) follows from the tower properties of expectations, and (b) follows from moving all $\mathbf{X}_t$ measurable random variables output the expectation $\mathbb{E}\left[ \cdot \mid \mathbf{X}_t\right]$.
\end{proof}

\section{Proof of Theorem~\ref{theorem:bound_residual}}

\label{proof:bound_residual}

The proof would be very similar to the proof of Theorem~\ref{theorem:bound_full_resolution}, except that we need to redefine the following variables.

\begin{align}
\label{eq:objective_residual_lhat}
\widehat{L}(\theta_0) = \frac{1}{NK}\sum_{k=1}^{K} \sum_{n=1}^{N} \left\| f(\mathbf{x}_{n,k}, t_{n,k}; \theta_0) + \sum_i B_iA_if(\mathbf{x}_{n,0},t_{n,k}; \theta_i) - \mathbf{x}_{n,0} \right\|_2^2 \nonumber
\end{align} 
\begin{align}
L(\theta_0) = \mathbb{E} \left[\left\| f(\mathbf{x}_{n,k}, t_{n,k}; \theta_0) + \sum_i B_iA_if(\mathbf{x}_{n,0},t_{n,k}; \theta_i) -   \mathbf{x}_{n,0} \right\|_2^2\right] \nonumber
\end{align} 
And the true objectives become:
\begin{align}
R(\theta_0) =  \mathbb{E}\left[\left\|f(\mathbf{x}_{n,k}, t_{n,k}; \theta_0) + \sum_i B_iA_if(\mathbf{x}_{n,0},t_{n,k}; \theta_i) - \mathbb{E}[\mathbf{X}_{ n,0}\mid \mathbf{x}_{n,k}, t_{n,k}] \right\|_2^2 \right] \nonumber
\end{align}
\begin{align}
\widehat{R}(\theta_0) = \sum_{n=1}^{N} \sum_{k=1}^{K} \frac{1}{NK}\left\|f(\mathbf{x}_{n,k}, t_{n,k}; \theta_0) + \sum_i B_iA_if(\mathbf{x}_{n,0},t_{n,k}; \theta_i) - \mathbb{E}[\mathbf{X}_{ n,0}\mid \mathbf{x}_{n,k}, t_{n,k}] \right\|_2^2  \nonumber 
\end{align}
\begin{align}
\mathcal{V}(\mathcal{S}) = \frac{1}{NK}\sum_{n=1}^{N}\sum_{k=1}^{K}\left\|E[\mathbf{X}_{n,0} \mid \mathbf{x}_{n,k}, t_{n,k}] + \sum_{i=1}^{I}B_iA_if(\mathbf{x}_{n,k}, t_{n,k};\theta_i)- \mathbf{x}_{n,0} \right\|_2^2 \nonumber
\end{align}

Radamacher complexity  $\mathfrak{R}_L(\mathcal{F}, \mathcal{S}_0)$, and $\mathfrak{R}_R(\mathcal{F}, \mathcal{S}_0)$ become
\begin{align}
\mathfrak{R}_L(\mathcal{F}, \mathcal{S}_0) = 
\mathbb{E}_{\mathcal{S}_0} \left[ \sup_{\theta \in \Theta} \frac{1}{NK} \sum_{n=1}^{N}\sum_{k=1}^{K} \sigma_{n,k}  \ell(\mathbf{X}_{n,k}, t_{n,k}; \theta)  \right] 
\end{align}
where
\begin{align}
\ell(\mathbf{X}_{n,k}, t_{n,k}; \theta) =  \left\|f(\mathbf{X}_{n,k}, t_{n,k}; \theta) + \sum_{i=1}^{I} B_i A_i f(\mathbf{X}_{n,0}, t_{n,0}; \theta_i)-  \mathbf{X}_{ n,0} \right\|_2^2
\end{align}
\begin{align}
\mathfrak{R}_R(\mathcal{F}, \mathcal{S}_0) = 
\mathbb{E}_{\mathcal{S}_0} \left[ \sup_{\theta \in \Theta} \frac{1}{NK} \sum_{n=1}^{N}\sum_{k=1}^{K} \sigma_{n,k}  \ell(\mathbf{X}_{n,k}, t_{n,k}; \theta)  \right] 
\end{align}
where
\begin{align}
\ell(\mathbf{X}_{n,k}, t_{n,k}; \theta) =  \left\|f(\mathbf{X}_{n,k}, t_{n,k}; \theta) + \sum_{i=1}^{I} B_i A_i f(\mathbf{X}_{n,k},t_{n,k}; \theta_i)- \mathbb{E}[\mathbf{X}_{ n,0}\mid \mathbf{X}_{n,k}, t_{n,k}] \right\|_2^2
\end{align}

\end{document}